\def\year{2018}\relax

\documentclass[letterpaper]{article} 
\usepackage{aaai18}  
\usepackage{times}  
\usepackage{helvet}  
\usepackage{courier}  
\usepackage{url}  
\usepackage{graphicx}  
\usepackage{dcolumn}
\usepackage{graphics}
\usepackage{helvet}
\usepackage{amssymb}
\usepackage{amsmath}
\usepackage{array}
\usepackage{float}
\usepackage{graphicx,epsfig,epstopdf}
\usepackage{psfrag}
\usepackage{subcaption}
\usepackage{tikz,rotating}
\usepackage{color}
\usepackage{theorem}
\usepackage{algorithm}
\usepackage[noend]{algpseudocode}
\usepackage{verbatim}
\usepackage{enumitem}
\usepackage{multirow}
\usepackage{mathabx}
\usepackage{xr}
\usepackage{mathrsfs}
\usepackage{thmtools, thm-restate}

\newtheorem{theorem}{Theorem}

\newtheorem{lemma}[theorem]{Lemma}
\newtheorem{definition}{Definition}
\newtheorem{example}{Example}

\DeclareMathOperator*{\argmax}{arg\,max}

\begin{document}
	
%
\title{Computing the Strategy to Commit to in Polymatrix Games \protect\\ (Extended Version)}

\author{Giuseppe De Nittis\textnormal{,} Alberto Marchesi \textnormal{and} Nicola Gatti\\ ~ Politecnico di Milano\\ ~ Piazza Leonardo da Vinci, 32\\ ~ Milano, Italy\\ ~ \{giuseppe.denittis, alberto.marchesi, nicola.gatti\}@polimi.it}

\maketitle

\begin{abstract}
Leadership games provide a powerful paradigm to model many real-world settings. Most literature focuses on games with a \emph{single} follower who acts \emph{optimistically}, breaking ties in favour of the leader. Unfortunately, for real-world applications, this is unlikely. In this paper, we look for efficiently solvable games with \emph{multiple} followers who play either optimistically or \emph{pessimistically}, i.e., breaking ties in favour or against the leader. 
We study the computational complexity of finding or approximating an optimistic or pessimistic leader-follower equilibrium in specific classes of succinct games---polymatrix like---which are equivalent to 2-player Bayesian games with uncertainty over the follower, with interdependent or independent types. Furthermore, we provide an exact algorithm to find a pessimistic equilibrium for those game classes. Finally, we show that in general polymatrix games the computation is harder even when players are forced to play pure strategies.
\end{abstract}

\section{Introduction}

Leadership games have recently received a lot of attention in the Artificial Intelligence literature, also thanks to their use in many real-world applications, e.g., security and protection~\cite{basilico2017adversarial,kar2017cloudy,kar2017trends}.
In principle, the paradigm is simple---one or more \emph{leaders} commit to a potentially mixed strategy, the \emph{followers} observe the commitments, and then they play their best-responses---, but it can be declined in many different ways.
The crucial issue is the computational study of the problem of finding the best leaders' strategy. 
In this paper, we provide new computational complexity results and algorithms for games with one leader and two or more followers.

\textbf{Related works}. 
In the 1-leader/1-follower case, we can distinguish different scenarios according to how the follower breaks ties (in favour to the leader---\emph{optimistic}---or against---\emph{pessimistic}) and the presence of uncertain information (Bayesian games). 
When the follower behaves pessimistically, the expected utility of the leader may not admit any \emph{maximum}, and the equilibrium corresponds to the \emph{supremum} of the utility~\cite{von2010leadership}. 
In this case, there is no leader's strategy where the value of the supremum is attained, so a strategy providing an approximation of such value must be computed.
While the literature has mainly focused on the optimistic case, it has been recently showed that the pessimistic case is of extraordinary importance in practice, since wrongly assuming the follower to be optimistic may lead to an arbitrary loss. 
This suggests that optimistic and pessimistic equilibria---being the extremes in terms of utility for the leader---should be considered together aiming to make a robust commitment. 

The computation of the equilibrium in the 1-leader/1-follower case requires polynomial time both in the optimistic~\cite{conitzer2006computing} and pessimistic case~\cite{von2010leadership}.
However, while the computation of an optimistic equilibrium is conceptually simple and can be done by solving a single linear program~\cite{conitzer2011commitment}, computing a pessimistic equilibrium is much more involved and requires a non-trivial theoretical study~\cite{von2010leadership}.
Conversely, in the presence of uncertainty, finding an optimistic equilibrium when the follower can be of a non-fixed number of types and the utility of the leader depends on the type of the follower (\emph{interdependent} types) is Poly-$\mathsf{APX}$-complete~\cite{letchford2009learning}.
The reduction does not apply to the simplified case in which the types are independent and not even to the computation of a pessimistic equilibrium, leaving these problems open.

The study of games with multiple followers is even more challenging. 
On one side, the equilibrium-computation problem is much more involved and largely unexplored. 
On the other side, many practical scenarios present multiple independent followers (e.g., pricing, toll-setting, and security).
In this case, the followers' game resulting from the leader's commitment can have different structures (e.g., followers can play sequentially or simultaneously). 
In this paper, we focus on games in which the followers play simultaneously, reaching a Nash Equilibrium given the leader's commitment. 
The problem of computing an optimistic or pessimistic equilibrium is not in Poly-$\mathsf{APX}$ even with two followers in polymatrix games~\cite{basilico2017methods}.
Furthermore, an optimistic equilibrium can be found using global optimization tools, whereas it is not known whether there is a finite mathematical programming formulation to find a pessimistic one~\cite{basilico2017bilevel}.
When restricting the followers to play pure strategies in generic normal-form games, there is an efficient algorithm to compute an optimistic equilibrium, while there is not for the pessimistic one unless $\mathsf{P}=\mathsf{NP}$ (note that the hardness is not due to the potential non-existence of the equilibrium)~\cite{coniglio2017pessimistic}.
These results suggest that, with multiple followers, computing a pessimistic equilibrium may be much harder than computing an optimistic one.

\textbf{Original contributions}. 
In this paper, we provide new results on the computation of leader-follower equilibria with multiple followers.
The motivation is to investigate whether there are game classes admitting efficient exact or approximation algorithms.
We identify two subclasses of polymatrix games such that, once fixed the number of followers, computing an optimistic or pessimistic equilibrium presents the same complexity, namely polynomial.
These classes are of practical interest, e.g., for security games.
Moreover, these games are equivalent to Bayesian games with one leader and one follower, where the latter may be of different types~\cite{howson1974bayesian}. 
In particular, our first game class is equivalent to Bayesian games with interdependent types, while the second game class is equivalent to Bayesian games with independent types (i.e., the leader's utility is independent of the follower's type). 
Thus, every result for a class also holds for its equivalent class.

We study if the problem keeps being easy when the number of followers is not fixed.
We show that there is not any polynomial-time algorithm to compute a pessimistic equilibrium, unless $\mathsf{P}=\mathsf{NP}$, and we provide an exact algorithm (conversely, to compute an optimistic equilibrium, one can adapt the algorithm provided in~\cite{conitzer2006computing} for Bayesian games with an optimistic follower, by means of our mapping). 
We also prove that, in all the instances where the pessimistic equilibrium is a supremum but not a maximum, an $\alpha$-approximation of the supremum can be found in polynomial time (also in the number of followers) for any $\alpha>0$, where $\alpha$ is the additive loss.
Furthermore, we show that this problem is Poly-$\mathsf{APX}$-hard, providing a single reduction for the optimistic and pessimistic cases even when the types are independent (this strengthens the result already known for Bayesian games with an optimistic follower and interdependent types). 
We also provide a simple approximation algorithm showing that these problems are in Poly-$\mathsf{APX}$ class. 
This shows that, in Bayesian games with uncertainty over the follower, computing a pessimistic equilibrium  is as hard as computing an optimistic equilibrium.

Finally, we investigate if general polymatrix games, in case the followers are restricted to play pure strategies, admit approximation algorithms. 
We provide a negative answer, showing that in the optimistic case the problem is not in Poly-$\mathsf{APX}$ if the number of followers is not fixed unless $\mathsf{P}=\mathsf{NP}$.
Moreover, we show that, unless $\mathsf{P}=\mathsf{NP}$, the problem is even harder in the pessimistic case, being not in Poly-$\mathsf{APX}$ even when the number of followers is fixed (conversely, it is known that an optimistic equilibrium can be computed efficiently if the number of followers is fixed).

\section{Problem Formulation}\label{sec:prob_form}
We study scenarios with one player acting as the \textit{leader} and with two or more players acting as \textit{followers}. 
Formally, let $N = \{1,2,\ldots,n\}$ be the set of players, where the $n$-th player is the leader and $F=N \setminus \{n\}$ is the set of followers. 
Each player $p$ has a set of actions $A_p = \{a_p^1,a_p^2,\ldots,a_p^{m_p}\}$, being $a_p^j$ the $j$-th action played by player~$p$ and $m_p=|A_p|$ the number of actions available to player~$p$.
Moreover, for each player $p$, let us define her \textit{strategy vector} (or \emph{strategy}, for short) as $s_p \in [0,1]^{m_p}$ with $\sum_{a_p \in A_p} s_p(a_p) = 1$, where $s_p(a_p)$ is the probability with which action $a_p$ is played by player $p$. 
We refer to the strategy space of player~$p$, which is the $(m_p-1)$-simplex, as $\Delta_p = \{s_p \in [0,1]^{m_p}:\sum_{a_p \in A_p} s_p(a_p) = 1\}$. 
We say that a strategy is \textit{pure} if only one action is played with strictly positive probability, otherwise it is called \textit{mixed}.
If all the strategies of the players are pure, each agent playing one single action, we compactly refer to the collection of all played actions, called \emph{action profile}, as $a=(a_1,a_2,\ldots,a_n)$, otherwise we denote with $s=(s_1,s_2,\ldots,s_n)$ the \textit{strategy profile}.

We focus on classes of games with specific structures.
\begin{definition}
	A Polymatrix Game (PG) is represented by a graph $G=(N,E)$ where:
	\begin{itemize}
		\item the players correspond to vertices of $G$;
		\item each player $p \in N$ plays against her neighbours, i.e., all the players $q$ such that $(p,q)\in E$;
		\item the utility $U_p : A_1 \times \ldots \times A_n \rightarrow \mathbb{R}$ of player $p$ is separable, i.e., for each edge $(p,q) \in E$, there is a game between $p$ and $q$ such that $U_{p,q}, U_{q,p} : A_p \times A_q \rightarrow \mathbb{R}$ define the payoffs of $p$ and of $q$,  respectively, in that game, and the total player's utility is given by $U_p(a_1,\ldots,a_n) = \sum\limits_{q:(p,q)\in E} U_{p,q}(a_p,a_q) $.\footnote{In the rest of the paper, we assume that both $U_{p,q}$ and $U_{q,p}$ are defined over $A_p \times A_q$, where $p$ smaller than~$q$.}
	\end{itemize}
\end{definition}
\begin{definition}
	A One-Level Tree Polymatrix Game (OLTPG) is a PG where the graph $G$ is a one-level tree composed of a root and some leaves directly connected to the root.
\end{definition}
Given an OLTPG, we call root-player that one associated with the tree root and leaf-players the other players.
\begin{definition}
	A Star Polymatrix Game (SPG) is an OLTPG s.t. for every couple of leaf-players $p,q \in N \setminus \{r\}$, where $r \in N$ is the root-player, $U_{r,p} = U_{r,q} = U_r$, with all the leaf-players sharing the same set of actions.
\end{definition}

In the following, we always assume that the root-player is $n$---the leader---, while the leaf-players are the players in $F$---the followers.
These special classes of games, OLTPGs and SPGs, are special cases of polymatrix games and are closely connected with many security scenarios. 
In fact, it often happens that different Attackers, acting as followers, do not influence each other's payoffs, having different preferences over the targets, e.g., when different groups of criminals attack different spots in the same city.
Moreover, we can model security applications as OLTPGs or SPGs, depending on the fact that the utility of the Defender, acting as the leader, is affected or not by the identity of the Attacker who performed the attack. 
From the Defender's perspective, it may be more important protecting the targets than knowing who committed the attack since the safety of people and buildings is the priority, as shown in Example~\ref{ex:security}.
\begin{example}\label{ex:security}
	An airport $a$, a bank $b$ and a church $c$ are targets for two local gangs, $\mathcal{A}_1$ and $\mathcal{A}_2$. 
	The buildings are protected by a guard $\mathcal{D}$, who patrols among such locations.
	The guard is the leader, who commits to a strategy by patrolling among the buildings, while the gangs are the followers, moving after having observed the Defender's commitment.
	All the players have the same actions, namely $a,b,c$: if both $\mathcal{D}$ and either $\mathcal{A}_1$ or $\mathcal{A}_2$ select the same location, the gang is caught; otherwise the crime is successful.
	If $\mathcal{D}$ is concerned with the type of gang she is facing, we can employ an OLTPG (see Figure~\ref{fig:oltpg}). 
	Conversely, if the Defender is only concerned about the protection of the buildings, her utility is the same, independently of the gang that attacked, and, thus, we can employ an SPG, as shown in Figure~\ref{fig:spg}.

	\begin{figure}[!htbp]
		\begin{subfigure}{.25\linewidth}
			\begin{scriptsize}
				\centering
				\begin{tabular}{ c  c | c | c | c |}
					& \multicolumn{4}{ c }{$\mathcal{A}_1$} \\ \cline{3-5}
					& & $a$ & $b$ & $c$ \\ \cline{2-5}
					\multirow{3}{*}{\rotatebox{90}{$\mathcal{D}$}}
					& \multicolumn{1}{ |c| }{$a$} & $9,0$ & $0,4$ & $0,6$ \\ \cline{2-5}
					& \multicolumn{1}{ |c| }{$b$} & $0,8$ & $5,0$ & $0,6$\\ \cline{2-5}
					& \multicolumn{1}{ |c| }{$c$} & $0,8$ & $0,4$ & $7,0$\\ \cline{2-5}
				\end{tabular}
			\end{scriptsize}
		\end{subfigure}
		\hspace{2cm} 	
		\begin{subfigure}{.25\linewidth}
			\begin{scriptsize}
				\centering
				\begin{tabular}{ c  c | c | c | c |}
					& \multicolumn{4}{ c }{$\mathcal{A}_2$} \\ \cline{3-5}
					& & $a$ & $b$ & $c$ \\ \cline{2-5}
					\multirow{3}{*}{\rotatebox{90}{$\mathcal{D}$}}
					& \multicolumn{1}{ |c| }{$a$} & $3,0$ & $0,8$ & $0,4$ \\ \cline{2-5}
					& \multicolumn{1}{ |c| }{$b$} & $0,6$ & $1,0$ & $0,4$ \\ \cline{2-5}
					& \multicolumn{1}{ |c| }{$c$} & $0,6$ & $0,8$ & $2,0$ \\ \cline{2-5}
				\end{tabular}
			\end{scriptsize}
		\end{subfigure}
		\caption{An instance of an OLTPG (the utility of $\mathcal{D}$ is different w.r.t. the Attacker she is facing).}
		\label{fig:oltpg}
	\end{figure}
	
	\begin{figure}[!htbp]
		\begin{subfigure}{.25\linewidth}
			\begin{scriptsize}
				\centering
				\begin{tabular}{ c  c | c | c | c |}
					& \multicolumn{4}{ c }{$\mathcal{A}_1$} \\ \cline{3-5}
					& & $a$ & $b$ & $c$ \\ \cline{2-5}
					\multirow{3}{*}{\rotatebox{90}{$\mathcal{D}$}}
					& \multicolumn{1}{ |c| }{$a$} & $9,0$ & $0,4$ & $0,6$\\ \cline{2-5}
					& \multicolumn{1}{ |c| }{$b$} & $0,8$ & $5,0$ & $0,6$\\ \cline{2-5}
					& \multicolumn{1}{ |c| }{$c$} & $0,8$ & $0,4$ & $7,0$ \\ \cline{2-5}
				\end{tabular}
			\end{scriptsize}
		\end{subfigure}
		\hspace{2cm} 	
		\begin{subfigure}{.25\linewidth}
			\begin{scriptsize}
				\centering
				\begin{tabular}{ c  c | c | c | c |}
					& \multicolumn{4}{ c }{$\mathcal{A}_2$} \\ \cline{3-5}
					& & $a$ & $b$ & $c$ \\ \cline{2-5}
					\multirow{3}{*}{\rotatebox{90}{$\mathcal{D}$}}
					& \multicolumn{1}{ |c| }{$a$} & $9,0$ & $0,8$ & $0,4$\\ \cline{2-5}
					& \multicolumn{1}{ |c| }{$b$} & $0,6$ & $5,0$ & $0,4$\\ \cline{2-5}
					& \multicolumn{1}{ |c| }{$c$} & $0,6$ & $0,8$ & $7,0$\\ \cline{2-5}
				\end{tabular}
			\end{scriptsize}
		\end{subfigure}
		\caption{An instance of an SPG (the utility of $\mathcal{D}$ is the same in both games).}
		\label{fig:spg}
	\end{figure}
\end{example}

Furthermore, OLTPGs and SPGs are equivalent to special Bayesian Games (BGs) with one leader and one follower, where the follower can be of different types. 
More precisely, OLTPGs are equivalent to BGs with interdependent types, the utility of the leader depending on the type of the follower, whereas SPGs are equivalent to BGs with independent types.
First, we provide the formal definition of the game classes.

\begin{definition}
	A BG with interdependent types (int-BG) consists of:
	\begin{itemize}
		\item two players, $l$ (the leader) and $f$ (the follower);
		\item a set of actions for each player, $A_l=\{a_l^1,\ldots,a_l^{m_l}\}$ and $A_f=\{a_f^1,\ldots,a_f^{m_f}\}$, respectively;
		\item a set of types for player $f$, $\Theta_f=\{\theta_1,\ldots,\theta_t\}$;
		\item players' utility functions, $U_l(a_l,a_f,\theta_f), U_f(a_l,a_f,\theta_f) : A_l \times A_f \times \Theta_f \rightarrow \mathbb{R}$, which specify the payoff player $l$, respectively $f$, gets when $l$ plays action $a_l$, $f$ plays action $a_f$, and $f$ is of type $\theta_f$;
		\item a probability distribution over types, where $\Omega(\theta_p)$ denotes the probability of $\theta_p \in \Theta_f$.
	\end{itemize}
\end{definition}
\begin{definition}
	A BG with independent types (ind-BG) is defined as an int-BG, except for player $l$'s utility function, here defined as $U_l(a_l,a_f) : A_l \times A_f \rightarrow \mathbb{R}$.
\end{definition}

The following theorem shows the equivalence among game classes (the proof follows from~\cite{howson1974bayesian}).

\begin{theorem}\label{thm:equivalence}
	There is a polynomial-time-computable function mapping any int-BG (ind-BG) to an OLTPG (SPG) and \emph{vice versa}, where:
	\begin{itemize}
		\item player $l$ in the int-BG (ind-BG) corresponds to the root-player of the OLTPG (SPG);
		\item type $t$ of player~$f$ in the int-BG (ind-BG) corresponds to a leaf-player of the OLTPG (SPG);
	\end{itemize}
	such that, given any strategy profile, the expected utility of each player in the OLTPG (SPG) and the corresponding player/type in the int-BG (ind-BG) are the same.
\end{theorem}

A \textit{pure Nash Equilibrium} (NE) consists of an action profile $a^*=(a_1^*,\ldots,a_n^*)$ such that $U_p(a_1^*,\ldots,a_n^*) \geq U_p(a_1,\ldots,a_n)$ for every player $p \in N$ and action profile $a = (a_1, \ldots, a_n)$ such that for all $q \in N\setminus \{p\}, a_q^*=a_q$ and $a_p^*\neq a_p$. 
In other words, no player can improve her utility by unilaterally deviating from the equilibrium by playing some other action $a_p \neq a_p^*$. 
A \textit{mixed Nash Equilibrium} is a strategy profile $s^*=(s_1^*,\ldots,s_n^*)$ such that no player can improve her utility by playing a strategy $s_p \neq s_p^*$, given that the other players play as prescribed by the equilibrium. 
We observe that a PG always admits at least one mixed NE, while a pure NE may not exist~\cite{howson1972equilibria}.

In this paper, we are concerned with the computation of an equilibrium where the leader commits to a mixed strategy and then the followers, after observing the leader's commitment, play a pure NE in the resulting game. 
Specifically, we study two variants of this equilibrium concept: one in which the followers play to maximize the leader's utility, called Optimistic Leader-Follower Equilibrium (O-LFE), and one where the followers play to minimize it, which we refer to as Pessimistic Leader-Follower Equilibrium (P-LFE). 
Formally, computing an O-LFE amounts to solve the following problem:

\begin{scriptsize}
	\begin{align*}
		& \max\limits_{s_n \in \Delta_n} \max\limits_{\substack{(a_1^*,\ldots,a_{n-1}^*) \in \\ A_1\times \ldots \times A_{n-1}}} \; \sum\limits_{a_n \in A_n} U_n(a_1^*,\ldots,a_{n-1}^*,a_n)s_n(a_n):\\
		& \forall p, a_p^* \in \argmax\limits_{a_p \in A_p}\left\lbrace \sum\limits_{a_n \in A_n} U_p(a_1^*,\ldots,a_p,\ldots,a_{n-1}^*,a_n)s_n(a_n) \right\rbrace,
	\end{align*}
\end{scriptsize}

\noindent while computing a P-LFE amounts to solve this other bilevel problem:

\begin{scriptsize}
	\begin{align*}
		& \sup\limits_{s_n \in \Delta_n} \min\limits_{\substack{(a_1^*,\ldots,a_{n-1}^*) \in \\ A_1\times \ldots \times A_{n-1}}} \; \sum\limits_{a_n \in A_n} U_n(a_1^*,\ldots,a_{n-1}^*,a_n)s_n(a_n):\\
		& \forall p, a_p^* \in \argmax\limits_{a_p \in A_p}\left\lbrace \sum\limits_{a_n \in A_n} U_p(a_1^*,\ldots,a_p,\ldots,a_{n-1}^*,a_n)s_n(a_n) \right\rbrace.
	\end{align*}
\end{scriptsize}

Notice that, when restricting the attention to OLTPGs, an outcome of the followers' game is an NE if each follower is best-responding to the leader's commitment.
Moreover, w.l.o.g., we can safely assume that each follower plays a pure strategy since, once the leader's strategy is fixed, the follower's utility function is linear in her strategy.

Let us observe that, since the equivalences in Theorem~\ref{thm:equivalence} are direct, all the computational results---including approximation results---holding for OLTPGs also hold for int-BGs while the results holding for SPGs also hold for ind-BGs, and \emph{vice versa}. As a consequence, the computation of an optimistic equilibrium in OLTPGs is Poly-$\mathsf{APX}$-complete~\cite{letchford2009learning}.

\section{Finding an Exact Pessimistic Equilibrium}

First, we state that computing a P-LFE in SPGs is $\mathsf{NP}$-hard and, \emph{a fortiori}, it is hard also in OLTPGs and PGs.
Moreover, using the mappings in Theorem~\ref{thm:equivalence}, the problem is $\mathsf{NP}$-hard also in int-BGs and ind-BGs.~\footnote{The result follows from a reduction of the maximum clique problem. For details, please see the proof of Theorem~\ref{thm:poly_apx_hardness}.}
\begin{restatable}{theorem}{NPhardness}\label{thm:nphardness}
	Computing a P-LFE in SPGs is $\mathsf{NP}$-hard.
\end{restatable}

Now, we provide an exact algorithm for computing a P-LFE in OLTPGs whose compute time is exponential in the number of followers and polynomial in the number of actions of the players. The algorithm extends the procedure given in~\cite{von2010leadership} to find a supremum of the leader's utility function with 2-player games, and it also includes a procedure to compute a strategy that allows the leader to achieve an $\alpha$-approximation (in additive sense) of the supremum when there is no maximum, for any $\alpha > 0$.

\begin{algorithm}
	\caption{Exact-P-LFE}
	\label{alg:exact}
	\scriptsize
	\begin{algorithmic}[1]
		\Function{Exact-P-LFE}{$\alpha$}
		\ForAll{$a = (a_1, \ldots, a_{n-1}) \in A_F$}
		\ForAll{$p \in F$} 
		\State $T_p := \{ a_p' \in A_p \mid U_p^{a_p} = U_p^{a_p'}\}$ 
		\EndFor
		\State $\epsilon^a := \Call{Solve-Emptyness-Check}{ \{ T_p \}_{p \in F}, a } $
		\If{$\epsilon^a > 0$}
		\State $(v^a, s_n^a, \zeta_{a_p'}^p) := \Call{Solve-Max-Min}{ \{ T_p \}_{p \in F}, a }$
		\State $\beta^a := | \{ \zeta_{a_p'}^p \mid \zeta_{a_p'}^p = 0 \} | > 0 $
		\EndIf
		\EndFor
		\State $a^* := \argmax_{a \in A_F} v^a$
		\If{$\beta^{a^*}$}
		\State \Return $\Call{Find-Apx}{ \{ T_p \}_{p \in F}, a^*, v^{a^*}, \alpha }$
		\EndIf
		\State \Return $s_n^{a^*}$
		\EndFunction
	\end{algorithmic}
\end{algorithm}

The algorithm is based on the enumeration of all the followers' action profiles, i.e., all the tuples $(a_1, \ldots, a_{n-1})$ belonging to the set $A_F = \bigtimes_{p \in F} A_p$, and, for each of them, it computes the best strategy the leader can commit to (under the pessimistic assumption) provided that $a_p$ is a best-response for follower $p$, for every $p \in F$. 
For ease of notation, given $a_p \in A_p$ with $p \in F$, let $U_p^{a_p} \in \mathbb{R}^{|A_n|}$ be a vector whose components are defined as $U_{p,n}(a_p,a_n)$, for $a_n \in A_n$. 
The complete algorithm procedure is detailed in Algorithm~\ref{alg:exact}, where it is assumed that the game elements can be accessed from any point, including sub-procedures, and the parameter $\alpha$ defines the quality of the approximation of the supremum, whenever a maximum does not exist.

At each iteration, the algorithm calls two sub-procedures that solve two LP programs. Specifically, $\Call{Solve-Emptyness-Check}{ \{ T_p \}_{p \in F}, a }$ computes the optimum of the following program:

\begin{scriptsize}
	\begin{align*}
		\max_{\substack{\epsilon \geq 0 \\ s_n \in \Delta_n}} & \quad \epsilon \quad\quad \text{s.t.}  \\
		& \hspace{-0.4cm} \sum_{a_n \in A_n} U_{p,n}(a_p,a_n) s_n(a_n) - \sum_{a_n \in A_n} U_{p,n}(a_p',a_n) s_n(a_n) - \epsilon \geq 0\\
		& \quad\quad\quad\quad\quad \forall a_p' \in A_p \setminus T_p, \forall p \in F; \\
	\end{align*}
\end{scriptsize}

\noindent  while $\Call{Solve-Max-Min}{ \{ T_p \}_{p \in F}, a }$ solves the following:

\begin{scriptsize}
	\begin{align*}
		\max_{\substack{s_n \in \Delta_n}} & \quad \sum_{p \in F} v_p \quad\quad \text{s.t.} \\
		& \hspace{-0.4cm}  v_p - \sum_{a_n \in A_n} U_{n,p}(a_p',a_n) s_n(a_n) \leq 0\quad \forall a_p' \in T_p, \forall p \in F \\
		& \hspace{-0.4cm}  \sum_{a_n \in A_n} U_{p,n}(a_p,a_n) s_n(a_n) - \hspace{-0.25cm} \sum_{a_n \in A_n} U_{p,n}(a_p',a_n) s_n(a_n) - \zeta_{a_p'}^p = 0\\
		& \quad\quad\quad\quad\quad \forall a_p' \in A_p \setminus T_p, \forall p \in F \\
		& \hspace{-0.4cm} \zeta_{a_p'}^p \geq 0 \quad \forall a_p' \in A_p \setminus T_p, \forall p \in F.
	\end{align*}
\end{scriptsize}

\noindent Finally, $\Call{Find-Apx}{ \{ T_p \}_{p \in F}, a^*, v^{a^*}, \alpha }$ employs the following LP program to find a leader's strategy providing an $\alpha$-approximation of the supremum:

\begin{scriptsize}
	\begin{align*}
		\max_{\substack{\epsilon \geq 0 \\ s_n \in \Delta_n}} & \quad \epsilon \quad\quad \text{s.t.} \\
		& \hspace{-0.4cm} \sum_{p \in F} v_p \geq v^{a^*} - \alpha \\
		& \hspace{-0.4cm} v_p - \sum_{a_n \in A_n} U_{n,p}(a_p',a_n) s_n(a_n) \leq 0\quad \forall a_p' \in T_p, \forall p \in F \\
		& \hspace{-0.4cm} \sum_{a_n \in A_n} U_{p,n}(a_p,a_n) s_n(a_n) - \hspace{-0.15cm} \sum_{a_n \in A_n} U_{p,n}(a_p',a_n) s_n(a_n) - \epsilon \geq 0\\
		& \quad\quad\quad\quad\quad \forall a_p' \in A_p \setminus T_p, \forall p \in F. \\
	\end{align*}
\end{scriptsize}

The following theorem shows that Algorithm~\ref{alg:exact} is correct.

\begin{theorem}
	Given an OLTPG, Algorithm~\ref{alg:exact} finds a P-LFE, and, whenever the leader's utility function does not admit a maximum, it returns an $\alpha$-approximation of the supremum.
\end{theorem}

\begin{proof}
	Before proving the statement, we introduce some useful notation. Given $a_p \in A_p$, with $p \in F$, let $\Delta_n(a_p)$ be the region of the leader's strategy space $\Delta_n$ containing those strategies $s_n$ such that follower $p$'s best-response to $s_n$ is $a_p$, i.e., $\Delta_n(a_p) = \{s_n \in \Delta_n \mid a_p \in \argmax_{a_p' \in A_p} \sum_{a_n \in A_n} U_{p,n}(a_p',a_n) s_n(a_n) \}$. Moreover, given a followers' action profile $a = (a_1, \ldots, a_{n-1}) \in A_F$, let $\Delta_n(a) = \bigcap_{p \in F} \Delta_n(a_p)$. We denote with $\Delta_n^o(\cdot)$ the interior of $\Delta_n(\cdot)$ relative to $\Delta_n$, and we call $\Delta_n(\cdot)$ full-dimensional if $\Delta_n^o(\cdot)$ is not empty. 
	
	In order to prove the result, we define the search problem of computing a P-LFE, as follows:
	\begin{align}\label{prob:pes_rewrite}
		\max_{a \in A_D} \max_{s_n \in \Delta_n(a)} \min_{\substack{a' \in A_F: \\ U_{p}^{a_p} = U_p^{a_p'} }} \sum_{p \in F} \sum_{a_n \in A_n} \hspace{-0.25cm} U_{n,p} (a_p',a_n) s_n(a_n),
	\end{align}

	where $A_D = \{ a \in A_F \mid \Delta_n(a) \text{ is full-dimensional } \}$.

	First, using a simple inductive argument, we derive a new definition for $\Delta_n$, which is as follows:
	\begin{align}\label{eq:delta}
		\Delta_n = \bigcup_{a \in A_D} \Delta_n(a)
	\end{align}
	Let us start noticing that $\Delta_n = \bigcup_{a \in A_F} \Delta_n(a)$. Then, take $a' \in A_F \setminus A_D$ and define $S = \Delta_n \setminus \bigcup_{a \in A_F \setminus \{a'\}} \Delta_n(a)$. We observe that $S$ is a subset of $\Delta_n(a')$, and, thus, it is also a subset of $\Delta_n^o(a')$, which is empty since $a' \notin A_D$, so $S$ is empty. Therefore, we can write $\Delta_n = \bigcup_{a \in A_F \setminus \{a'\}} \Delta_n(a)$, which we use as new definition for $\Delta_n$. Iterating in this manner until all the elements in $A_F \setminus A_D$ have been considered, we eventually obtain the result.
	
	Second, we recall a result from~\cite{von2010leadership}, i.e., for every $a_p, a_p' \in A_p$, it holds:
	\begin{align}\label{eq:imply}
		s_n \in \Delta_n^o(a_p) \wedge s_n \in \Delta_n(a_p') \implies U_{p}^{a_p} = U_p^{a_p'}.
	\end{align}
	
	We are now ready to prove Equation~(\ref{prob:pes_rewrite}), as follows:
	\begin{align*}
		V & = \sup_{s_n \in \Delta_n} \min_{\substack{a' \in A_F: \\ s_n \in \Delta_n(a')}} \sum_{p \in F} \sum_{a_n \in A_n} \hspace{-0.25cm} U_{n,p}(a_p',a_n) s_n(a_n) \\
		& = \max_{a \in A_D} \sup_{s_n \in \Delta_n(a)} \min_{\substack{a' \in A_F: \\ s_n \in \Delta_n(a')}} \sum_{p \in F} \sum_{a_n \in A_n} \hspace{-0.25cm} U_{n,p}(a_p',a_n) s_n(a_n),
	\end{align*}
	where the first equality directly follows from the definition of the problem, while the second one is obtained rewriting $\Delta_n$ as given by~(\ref{eq:delta}). Restricting $\Delta_n(a)$ to $\Delta_n^o(a)$ and using~(\ref{eq:imply}), we obtain:
	\begin{align*}
		&V \geq \max_{a \in A_D} \sup_{s_n \in \Delta_n^o(a)} \min_{\substack{a' \in A_F: \\ s_n \in \Delta_n(a')}} \sum_{p \in F} \sum_{a_n \in A_n} \hspace{-0.25cm} U_{n,p}(a_p',a_n) s_n(a_n) \\
		&\hspace{-0.15cm} = \max_{a \in A_D} \sup_{s_n \in \Delta_n^o(a)} \min_{\substack{a' \in A_F: \\ U_{p}^{a_p} = U_p^{a_p'} }} \sum_{p \in F} \sum_{a_n \in A_n} \hspace{-0.25cm} U_{n,p}(a_p',a_n) s_n(a_n) \\
		& \hspace{-0.15cm} = \max_{a \in A_D} \sup_{s_n \in \Delta_n(a)} \min_{\substack{a' \in A_F: \\ U_{p}^{a_p} = U_p^{a_p'} }} \sum_{p \in F} \sum_{a_n \in A_n} \hspace{-0.25cm} U_{n,p}(a_p',a_n) s_n(a_n) \\
		& \hspace{-0.15cm} \geq \max_{a \in A_D} \hspace{-0.05cm} \sup_{s_n \in \Delta_n(a)} \hspace{-0.05cm} \min_{\substack{a' \in A_F: \\ s_n \in \Delta_n(a')}} \sum_{p \in F}  \sum_{a_n \in A_n} \hspace{-0.25cm} U_{n,p}(a_p',a_n) s_n(a_n) \hspace{-0.05cm} = \hspace{-0.05cm} V,
	\end{align*}
	where the last equality holds since the minimum is taken over a finite set of linear functions and it is continuous, while the last inequality comes from the fact that the minimum is taken over a larger set of elements.
	Hence, all the inequalities must hold as equalities, which proves Equation~(\ref{prob:pes_rewrite}).
	
	The algorithm exploits Equation~(\ref{prob:pes_rewrite}) to compute a P-LFE. 
	Notice that, if $\Delta_n(a)$ is not full-dimensional, then $\Call{Solve-Emptyness-Check}{ \{ T_p \}_{p \in F}, a }$ returns zero, as, if there is no strategy $s_n \in \Delta^o(a)$, then there is always at least one inequality in the LP program which can be satisfied only by setting $\epsilon=0$.
	The algorithm iterates over all the followers' action profiles in $A_D$, as every $a \in A_F \setminus A_D$ is discarded since $\epsilon^a = 0$. 
	Then, for each remaining action profile, it solves the max-min expression on the right of Eq.~(\ref{prob:pes_rewrite}), which can be done with the LP program solved by $\Call{Solve-Max-Min}{ \{ T_p \}_{p \in F}, a }$.
	Finally, the algorithm selects the followers' action profile associated with the highest max-min expression value.
	
	In conclusion, note that, given some $a \in A_D$, $\beta^a$ is true if and only if $s_n^a$ is such that there is at least one follower $p$ who has a best-response $a_p'$ that is not in $T_p$, i.e., at least one variable $\zeta_{a_p'}^p$ is zero. 
	Thus, if $\beta^{a^*}$ is true, the leader's utility function does not admit a maximum, since for $s_n^{a^*}$ there is some follower who can play a best-response which is worse than the one played in $a^*$ in terms of leader's utility.
	If that is the case, $\Call{Find-Apx}{ \{ T_p \}_{p \in F}, a^*, v^{a^*}, \alpha }$ finds an $\alpha$-approximation of the supremum $v^{a^*}$ by looking for a strategy $s_n \in \Delta^o(a^*)$, with the additional constraints imposing that the leader's utility (in the pessimistic case) does not fall below $v^{a^*} - \alpha$.
	Such approximation always exists since $\Delta^o(a^*)$ is non-empty and the leader's utility is the minimum of a finite set of affine functions.
	\hfill $\Box$
\end{proof}

\textbf{Discussion}. Even though, as described next, one can adopt the algorithm proposed in~\cite{von2010leadership} to find a P-LFE in an OLTPG, this would result in a procedure that is more inefficient than Algorithm~\ref{alg:exact}. Indeed, one should first transform an OLTPG into an int-BG, by means of the mapping provided in Theorem~\ref{thm:equivalence}, and, then, cast the resulting game in normal form. However, this would require the solution of an exponential number of LP programs, each with an exponential number of constraints, since the number of actions of the resulting normal-form game is exponential in the size of the original game. 
Conversely, Algorithm~\ref{alg:exact} exploits the separability of players' utilities, avoiding the explicit construction of the normal form before the execution of the algorithm. As a result, our algorithm still requires the solution of an exponential number of LP programs, but each with a polynomial number of constraints.
Notice that avoiding the explicit construction of the normal form also allows the execution of Algorithm~\ref{alg:exact} in an anytime fashion, stopping the algorithm whenever the available time is expired.

\textbf{Experimental evaluation}. We ran Algorithm~\ref{alg:exact} on a testbed of OLTPGs, evaluating the running time as a function of the number of players $n$ and the number of actions per player $m$. Specifically, for each pair $(n,m)$, times are averaged over 20 game instances, with $n \in \{3, \ldots, 10\}$ and $m \in \{4, 6, \ldots\, 10, 15, \ldots, 70 \}$.
Game instances have been randomly generated, with each payoff uniformly and independently drawn from the interval $[0,100]$.
All experiments are run on a UNIX machine with a total of 32 cores working at 2.3 GHz, and equipped with 128 GB of RAM. Each game instance is solved on a single core, within a time limit of 7200 seconds.
The algorithm is implemented in Python~2.7, while all LP programs are solved with GUROBI~7.0, using the Python interface.
Figure~\ref{fig:plots} contains two plots of the average computing times, as a function of $n$ and $m$, respectively.

\begin{figure}[htbp]
	\begin{center}
		\includegraphics[scale=0.33]{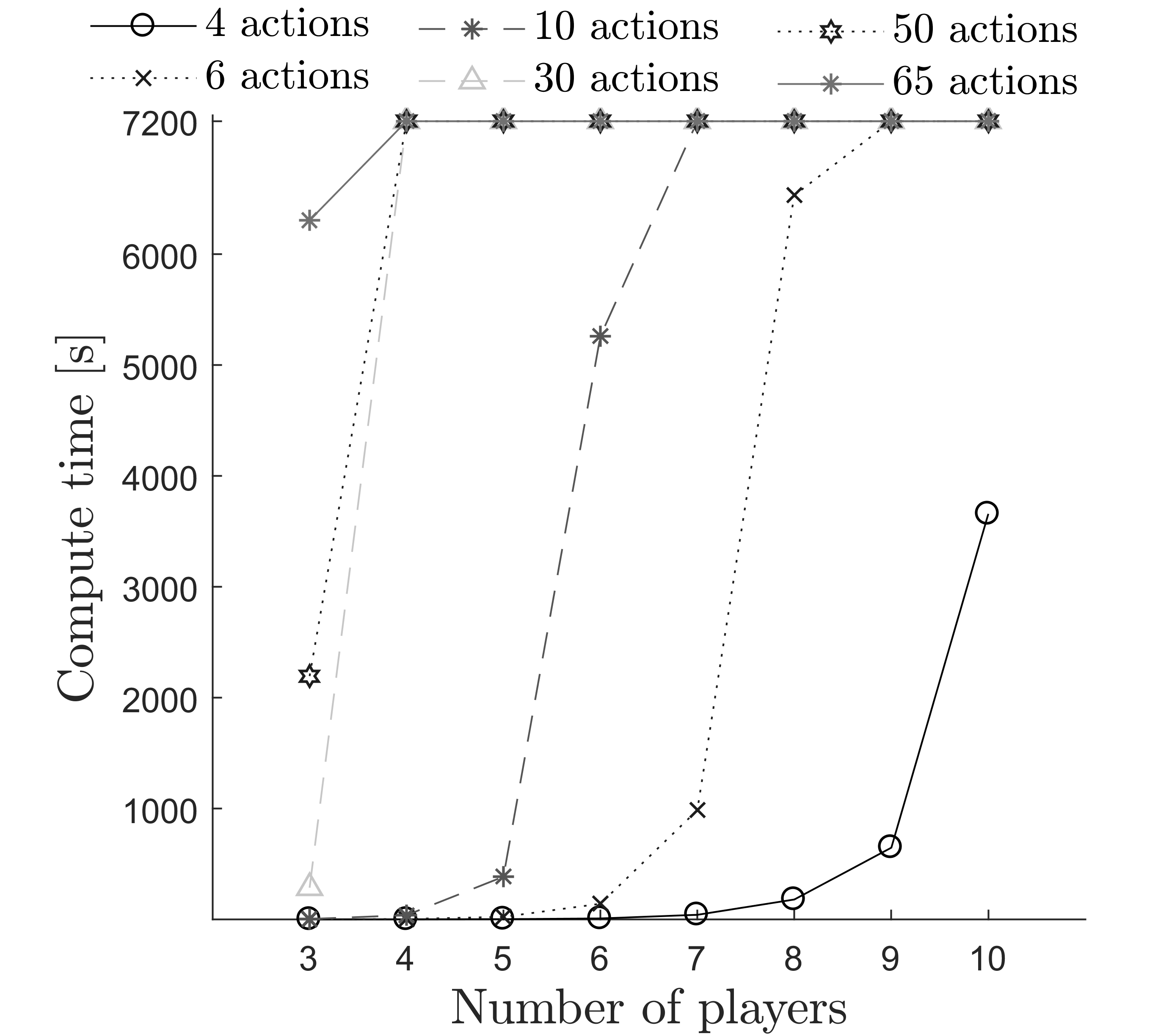}
		\includegraphics[scale=0.052]{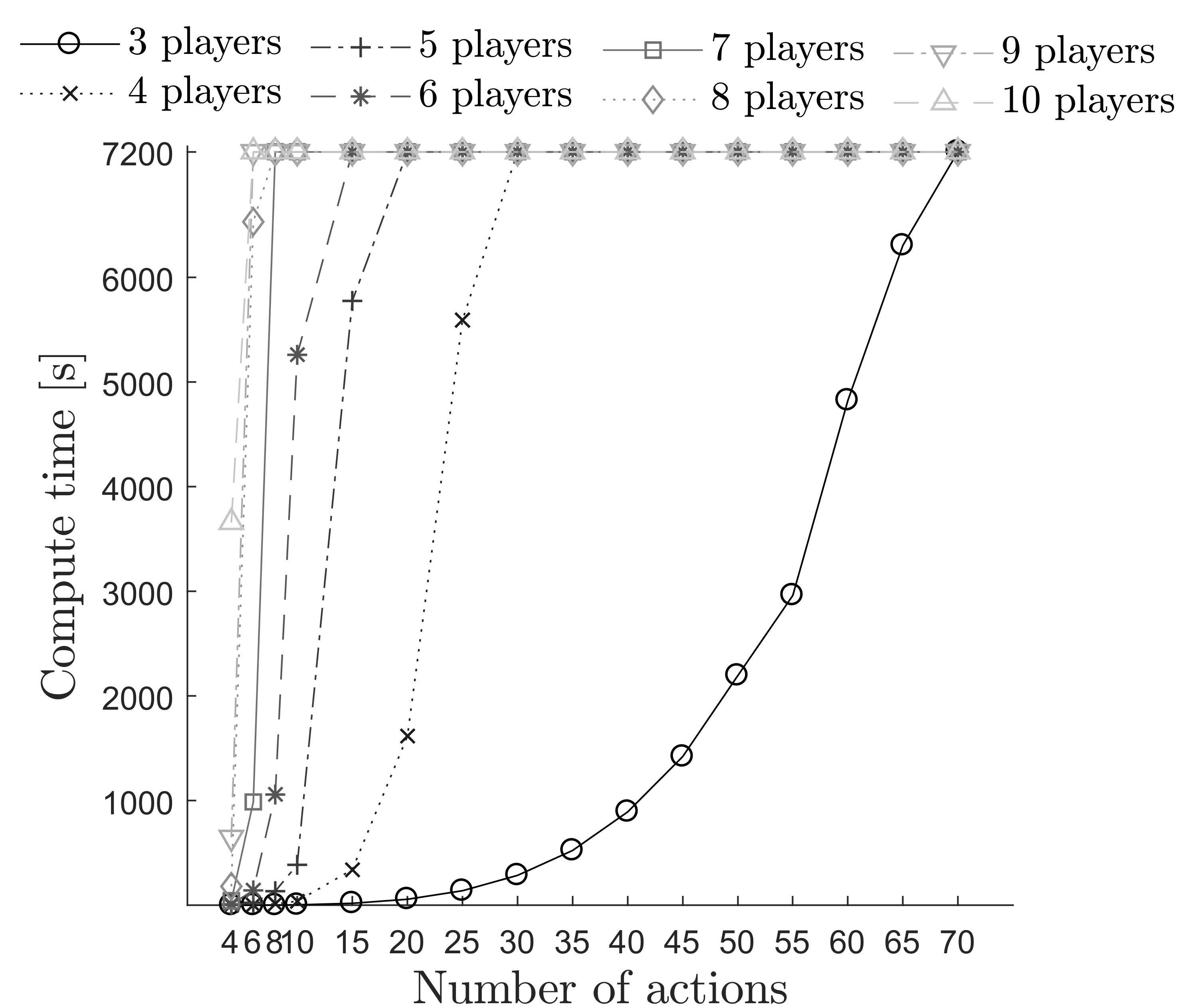}
	\end{center}
	\caption{Average computing times (in seconds), as a function of the number of players $n$ (on the left), and as a function of the number of actions per player $m$ (on the right).}
	\label{fig:plots}
\end{figure}

We observe that, as expected, the computing time increases exponentially in the number of players $n$, while, once $n$ is fixed, the growth is polynomial in the number of actions $m$. 
Specifically, the algorithm is able to solve within the time limit instances with 3 players, up to within $m=65$, while, as the number of players increases, the scalability w.r.t. $m$ decreases considerably, e.g., with 10 players, the algorithm can solve games with at most $m=4$.

\section{Approximating a Pessimistic Equilibrium} \label{sec:apx}

Initially, we study the computational complexity of approximating a pessimistic equilibrium.

\begin{theorem}\label{thm:poly_apx_hardness}
	Finding a P-LFE in SPGs is Poly-$\mathsf{APX}$-hard.
\end{theorem}

\begin{proof}
	We employ an approximation-preserving reduction from the maximum clique optimization problem, which is known to be Poly-$\mathsf{APX}$-hard~\cite{zuckerman2006linear}.
	
	\begin{definition}[MAXIMUM-CLIQUE (MC)]
		Given an undirected graph $G=(V,E)$, find a maximum clique of $G$, i.e., a complete sub-graph of $G$ with maximum size.
	\end{definition}
	
	First, we provide a polynomial mapping from MC to the problem of finding a P-LFE, reducing an arbitrary instance of MC to an SPG, and, then, we prove that the correspondence among instances is correct and the mapping is approximation-preserving.
	Letting $V = \{v_1, \ldots, v_r\}$, for every $v_p \in V$, we introduce a follower $p$, i.e., $N = \{1,\ldots,r,n\}$ with $n=r+1$. Each follower has two actions, i.e., $A_p = \{ a_0, a_1 \}$ for all $p \in F$, while the leader has an action per vertex, i.e., $A_n = \{a_n^1, \ldots, a_n^r\}$. Utilities are defined as follows:
	\begin{itemize}
		\item $U_{p,n}(a_0,a_n^i) = 1+r^2, $ for all $ (v_p,v_i) \notin E$;
		\item $U_{p,n}(a_0,a_n^i) = 1, $ for all $ (v_p,v_i) \in E$;
		\item $U_{p,n}(a_1,a_n^p) = r, $  for all $ v_p \in V$;
		\item $U_{p,n}(a_1,a_n^i) = 0, $ for all $ v_p, v_i \in V,$ with $ p\neq i$;
		\item $U_n(a_0,a_n^i) = 0, $ for all $ v_i \in V$;
		\item $U_n(a_1,a_n^i) = 1, $ for all $ v_i \in V$.
	\end{itemize} 

	Suppose that the graph $G$ admits a clique $C$ of size $J$. 
	W.l.o.g. we assume $J < r$ (the number of vertices of $G$), since instances with a maximum clique of size $r$ can be safely ruled out as we can check if the graph is complete in polynomial time. 
	Consider a mixed strategy of the leader such that each $a_n^i$ with $v_i \in C$ is played with probability equal to $\frac{1}{J}$. Then, each follower $p$ with $v_p \in C$ plays $a_1$: in fact, playing $a_1$, they get a utility of $\frac{r}{J} > 1$, while playing $a_0$ they can only get $1$, since no $a_n^i$ with $ (v_p,v_i) \notin E$ is ever played by the leader, being $C$ a clique.
	Therefore, the leader gets a utility of $|C|=J$ by playing such strategy.
	
	Suppose that, in a P-LFE of the SPG, the leader gets a utility equal to $J$ and, thus, given the definition of the game, there are exactly $J$ followers who play action $a_1$. Let us call $C$ the subset of vertices $v_p$ such that follower $p$ plays $a_1$: we prove that $C$ is a clique. In order for follower $p$ to play $a_1$ instead of $a_0$, the leader must play $a_n^p$ with probability greater than or equal to $\frac{1}{r}$, otherwise the follower would get a higher utility by playing $a_0$. Moreover, the leader cannot play any action $a_n^i$ such that $(v_p,v_i)\notin E$ with probability at least $\frac{1}{r}$, because otherwise the follower would play $a_0$, getting a utility greater than or equal to $1 + \frac{1}{r} \cdot r^2 = 1 + r$, which is clearly strictly greater than $r$ that is the maximum utility she can get by playing action $a_1$.
	Thus, the leader must play all the $J$ actions $a_n^p$ such that $v_p \in C$ with probability at least $\frac{1}{r}$, and there is no pair of vertices $v_p, v_i \in C$ such that $(v_p,v_i) \notin E$.
	So, the vertices in $C$ are completely connected, meaning that $C$ is a clique of size $J$.
	
	The reduction is approximation-preserving since the leader's utility coincides with the cardinality of the clique. Thus, given that MC is Poly-$\mathsf{APX}$-hard, the result follows. 
	Notice that the reduction works in both the optimistic and the pessimistic cases, as there is no follower who is indifferent among multiple best-responses.
	\hfill $\Box$
\end{proof}

Now, we provide a polynomial-time approximation algorithm for the P-LFE finding problem that guarantees an approximation factor polynomial in the size of the game, thus showing that the problem belongs to the Poly-$\mathsf{APX}$ class.

\begin{theorem}\label{thm:in_poly_apx}
	Computing a P-LFE in OLTPGs is in Poly-$\mathsf{APX}$.
\end{theorem}

\begin{proof}
	To prove the result, we provide an algorithm $\mathcal{A}$ working as follows. 
	First, $\mathcal{A}$ makes the leader play a 2-player leadership game against each follower independently.
	Let $U_{n,p}$ be the utility the leader gets in the game played against follower $p \in F$.
	Then, the algorithm selects the leader's strategy which is played against a follower $p$ such that $U_{n,p}$ is maximum. 
	The utility the leader gets adopting the strategy computed by means of algorithm $\mathcal{A}$ is equal to $U_n^{APX} \geq \max_{p \in F} U_{n,p}$, while the utility she would get in a P-LFE is equal to $U_n^{OPT} \leq (n-1) \cdot \max_{p \in F} U_{n,p}$.
	Thus, algorithm $\mathcal{A}$ guarantees an approximation factor equal to  $\frac{U_n^{APX}}{U_n^{OPT}} \geq \frac{\max_{p \in F} U_{n,p}}{(n-1) \cdot \max_{p \in F} U_{n,p}} = \frac{1}{n-1} = \frac{1}{O(n)}$. This concludes the proof. \hfill $\Box$
\end{proof}

The next result directly follows from Theorems~\ref{thm:poly_apx_hardness} and~\ref{thm:in_poly_apx}.

\begin{theorem}\label{cor:poly_apx_complete}
	Computing a P-LFE in OLTPGs is Poly-$\mathsf{APX}$-complete.
\end{theorem}

\section{Equilibrium Approximation in General Polymatrix Games}

From the previous sections, we know that the problem of computing an O/P-LFE is Poly-$\mathsf{APX}$-complete when instances are restricted to a specific class of games, namely OLTPGs.
In this section, we investigate the approximability of the problem of computing an O/P-LFE in PGs, when the followers are restricted to play pure strategies.

First, we prove that, when the number of followers is non-fixed, computing an O-LFE in PGs is not in Poly-$\mathsf{APX}$ unless $\mathsf{P}$ = $\mathsf{NP}$, and, thus, there is no polynomial-time approximation algorithm providing good (as the size of the input grows) approximation guarantees.

\begin{theorem}
	The problem of computing an O-LFE in PGs is not in Poly-$\mathsf{APX}$, unless $\mathsf{P}$ = $\mathsf{NP}$.
\end{theorem}

\begin{proof}
	We provide a reduction from 3-SAT.
	
	{\em Mapping.} Given a 3-SAT instance, i.e., a set of variables $V = \{v_1, \ldots, v_r\}$ and a set of 3-literal clauses $C = \{\phi_1, \ldots, \phi_s\}$, we build a PG with $n = s+1$ players, as follows. The set of players is $N = \{1, \ldots, s, n\}$, where the first $s$ players, the followers, are associated with the clauses in $C$, i.e., letting $F = \{1, \ldots, s\}$, follower $p \in F$ corresponds to $\phi_p \in C$. The leader (player $n$) has an action for each variable in $V$, plus an additional one, i.e., $A_n = \{ a_{v_1}, \ldots, a_{v_r}, a_w \}$ (where $w \notin V$). On the other hand, each follower has only four actions, namely $A_p = \{ a_0, a_1, a_2, a_3 \}$ for every $p \in F$. For any clause $\phi_p \in C$, with $\phi_p = l_1 \vee l_2 \vee l_3$, the payoffs of the corresponding follower $p$ are so defined:
	\begin{itemize}
		\item $U_{p, n}(a_i,a_v) = r+1$ if $v = v(l_i)$ and $l_i$ is positive, for every $i \in \{1,2,3\}$ (where $v(l_i)$ denotes the variable of $l_i$);
		\item $U_{p, n}(a_i,a_v) = 0$ if $v \neq v(l_i)$ and $l_i$ is positive, for every $i \in \{1,2,3\}$;
		\item $U_{p, n}(a_i,a_v) = 0$ if $v = v(l_i)$ and $l_i$ is negative, for every $i \in \{1,2,3\}$;
		\item $U_{p, n}(a_i,a_v) = \frac{r+1}{r}$ if $v \neq v(l_i)$ and $l_i$ is negative, for every $i \in \{1,2,3\}$;
		\item $U_{p, n}(a_0,a_n) = 0$ for every $a_n \in A_n$;
		\item $U_{p, q}(a_p, a_q) = 0$ for $a_p \in A_p \setminus \{a_0\}$ and $a_q \in A_q$, for every $q \in F \setminus \{p\}$;
		\item $U_{p, q}(a_0, a_q) = \frac{1}{s-1}$ for $a_q \in A_q \setminus \{a_0\}$, for every $q \in F \setminus \{p\}$;
		\item $U_{p, q}(a_0, a_0) = r+1$ for every $q \in F \setminus \{p\}$;
	\end{itemize}
	The leader's payoffs are defined as follows:
	\begin{itemize}
		\item $U_{n, p}(a_p,a_n) = \frac{1}{s}$ for every $a_n \in A_n$, $a_p \in A_p \setminus \{a_0\}$, and $p \in F$;
		\item $U_{n, p}(a_0,a_n) = \frac{\epsilon}{s}$ for every $a_n \in A_n$,
	\end{itemize}
	where $\epsilon > 0$ is an arbitrarily small positive constant. In the following, for ease of presentation and with abuse of notation, we define $U_{p,n}(a_p,s_n)$ as the utility follower $p \in F$ expects to obtain by playing against the leader, when the latter plays strategy $s_n \in \Delta_n$, i.e., $U_{p,n}(a_p,s_n) = \sum_{a_n \in A_n} U_{p,n}(a_p,a_n) \,\, s_n(a_n)$. Furthermore, given a truth assignment to the variables $T : V \rightarrow \{0,1\}$, let us define $s(T)$ as the set of leader's strategies $s_n \in \Delta_n$ such that $s_n(a_v) > \frac{1}{r+1}$ if $T(v) = 1$, while $s_n(a_v) < \frac{1}{r+1}$ whenever $T(v) = 0$. Clearly, no matter the truth assignment $T$, the set $s(T)$ is always non-empty, as one can make the probabilities in the strategy $s_n$ sum up to one by properly choosing $s_n(a_w)$. On the other hand, given a leader's strategy $s_n \in \Delta_n$, we define $T^{s_n}$ as the truth assignment in which $T^{s_n}(v) = 1$ if $s_n(a_v) > \frac{1}{r+1}$, while $T^{s_n}(v) = 0$ whenever $s_n(a_v) < \frac{1}{r+1}$ (the case $s_n(a_v) = \frac{1}{r+1}$ deserves a different treatment, although the proof can be easily extended to take it into consideration, we omit it for simplicity).
	Finally, without loss of generality, let us assume $s \geq 3$.
	
	Initially, we introduce the following lemma.
	
	\begin{lemma}\label{lemma}
		For any leader's strategy $s_n \in \Delta_n$, there exists an action $a_p \in A_p \setminus \{a_0\}$ such that $U_{p,n}(a_p,s_n) > 1$ if and only if $\phi_p$ evaluates to true under $T^{s_n}$.
	\end{lemma}
	
	\begin{proof}
		Suppose that $T^{s_n}$ makes $\phi_p = l_1 \vee l_2 \vee l_3$  true, and let $l_i$ be one of the literals that evaluate to true in $\phi_p$ (at least one must exist). Clearly, given the definition of $T^{s_n}$, $s_n(a_v) > \frac{1}{r+1}$ if $l_i$ is positive, whereas $s_n(a_v) < \frac{1}{r+1}$ when $l_i$ is negative. Two cases are possible. If $l_i$ is positive, then $U_{p,n}(a_i,s_n) = s_n(a_v) \cdot (r+1) > 1$, while, if $l_i$ is negative we have $U_{p,n}(a_i,s_n) = (1-s_n(a_v))\cdot \frac{r+1}{r} > 1$. Thus, $a_i \in A_p \setminus \{a_0\}$ is the action we are looking for.
		
		Now, let us prove the other way around. Suppose $a_p \in A_p \setminus \{a_0\}$ is such that $U_{p,n}(a_p,s_n) > 1$ and consider the case in which $a_p = a_i$ and literal $l_i$ is positive in $\phi_p$ (similar arguments also hold for the case where $l_i$ is negative). Letting $v=v(l_i)$, it easily follows that $s_n(a_v) \cdot (r+1) > 1$, implying that $s_n(a_v) > \frac{1}{r+1}$. Thus, given the definition of $T^{s_n}$, $\phi_p$ must evaluate to true.
		\hfill $\Box$
	\end{proof}
	
	{\em YES-instance.} Suppose that the given 3-SAT instance has a YES answer, i.e., there exists a truth assignment $T$ that satisfies all the clauses. We prove that, if this is the case, then in an O-LFE the leader gets a utility of $1$.
	Consider a leader's strategy $s_n \in s(T)$ and a followers' action profile $a \in \bigtimes_{p \in F} A_p$ where follower $p$'s action $a_p$ is such that $a_p = a_i$ and literal $l_i$ of $\phi_p$ evaluates to true under truth assignment $T$. Clearly, the action profile is always well-defined since $T$ satisfies all the clauses. Moreover, when there are many possible choices for action $a_p$, we assume that the follower plays the one providing her with the maximum utility given $s_n$.
	Now, we prove that $a$ is a pure NE in the followers' game resulting from the leader's commitment to $s_n$.
	Let $p \in F$ be a follower. Clearly, the follower's expected utility in action profile $a$ is $U_{p,n}(a_p,s_n)$ since she gets $0$ by playing against the other followers.
	The follower could deviate from $a_p$ in two different ways, either by playing an action corresponding to a different literal in the clause or by playing $a_0$. In the first case, the follower cannot get more than what she gets by playing $a_p$, given the definition of $a_p$. In the second case, the follower gets $(s - 1) \cdot \frac{1}{s - 1} = 1$, which is the utility obtained by playing against the other followers.
	Observing that $T$ is actually the same as $T^{s_n}$ and using Lemma~\ref{lemma}, we conclude that $U_{p,n}(a_p,s_n) > 1$ and no follower has an incentive to deviate from $a$, which makes it a pure NE given $s_n$.
	Finally, since we are in the optimistic case, the followers always play $a$ since it is the NE maximizing the leader's utility, as, in it, the leader gains $s \cdot \frac{1}{s} = 1$, which is the maximum payoff she can get. Moreover, for the same reason, the leader's utility in an O-LFE is~$1$.
	
	{\em NO-instance.} Suppose the 3-SAT instance has a NO answer, i.e., there is no truth assignment which satisfies all the clauses. 
	First, we prove that the followers' action profile $a \in \bigtimes_{p \in F} A_p$ in which all the followers play $a_0$ is a pure NE, no matter the leader's strategy $s_n$. 
	In $a$, every follower gets a utility of $(s-1) \cdot (r+1)$ which does not depend on the leader's strategy. Now, suppose that follower $p \in F$ deviates from $a$ by playing some action $a_p \neq a_0$, then she would get $U_{p,n}(a_p,s_n) \leq r + 1$, which is clearly strictly less than $r \cdot (s-1)$ given the assumption $s \geq 3$. Hence, $a$ is always a pure NE in the followers' game and it provides the leader with a utility of $s \cdot \frac{\epsilon}{s}= \epsilon$.
	Finally, we show that, for all leader's strategies $s_n \in \Delta_n$, there cannot be other NEs in the followers' game, and, thus, $a$ is the unique NE the followers can play.
	Let us start proving that all the action profiles in which some followers play $a_p \neq a_0$ and some others play $a_0$ cannot be NEs. 
	Let $p \in F$ be a follower such that $a_p \neq a_0$. Clearly, $p$ has an incentive to deviate by playing $a_0$ since $U_{p,n}(a_p,s_n) \leq r+1 < \sharp_i \cdot \frac{1}{s-1} + \sharp_0 \cdot (r+1)$ given that $\sharp_0 \geq 1$, where $\sharp_i$ is the number of followers other than $p$ who are playing $a_p \neq a_0$ and $\sharp_0$ is the number of followers playing $a_0$.
	In conclusion, it remains to prove that the followers' action profile in which they all play actions $a_p \neq a_0$ cannot be an NE. Let $p \in F$ be a follower such that $\phi_p$ is false under truth assignment $T^{s_n}$ (she must exist, as, otherwise, the 3-SAT instance would have answer YES). Clearly, $p$ has incentive to deviate playing $a_0$ since, using Lemma~\ref{lemma}, $U_{p,n}(a_p,s_n) < 1 = (s-1) \cdot \frac{1}{s-1}$.
	Therefore, in an O-LFE, the leader must get a utility of $\epsilon$.
	
	{\em Contradiction.} Suppose there exists a polynomial-time approximation algorithm $\mathcal{A}$ with approximation factor $r = \frac{1}{f(n)}$, where $f(n)$ is any polynomial function of $n$. Moreover, let us fix $\epsilon = \frac{1}{2^n}$ (notice that the polynomiality of the reduction is preserved, as $\epsilon$ can still be represented with a number of bits polynomial in $n$). If the 3-SAT instance has answer YES, then $\mathcal{A}$, when applied to the corresponding polymatrix game, must return a solution with value greater than or equal to $\frac{1}{f(n)} > \epsilon$. Instead, if the answer is NO, $\mathcal{A}$ must return a solution of value $\frac{\epsilon}{f(n)} < \epsilon$. Thus, the existence of $\mathcal{A}$ would imply that 3-SAT is solvable in polynomial time (the answer is YES if and only if the returned solution has value greater than $\epsilon$), which is an absurd, unless $\mathsf{P}$ = $\mathsf{NP}$. 
	\hfill $\Box$
\end{proof}

Finally, we show that approximating a P-LFE in PGs is harder than approximating an O-LFE, the problem being not in Poly-$\mathsf{APX}$ even when the number of followers is fixed.

\begin{theorem}\label{thm:not_poly_pes}
	Computing a P-LFE in PGs is not in Poly-$\mathsf{APX}$ even when $n=4$, unless $\mathsf{P}$ = $\mathsf{NP}$.
\end{theorem}

\begin{proof}
	We provide a reduction from 3-SAT.
	
	{\em Mapping.} Given a 3-SAT instance, i.e., $V = \{v_1, \ldots, v_r\}$ and $C = \{\phi_1, \ldots, \phi_s\}$, we build a PG with $n = 4$ players, as follows. The leader (player $4$) has an action for each variable in $V$, plus an additional one, i.e., $A_4 = \{ a_{v_1}, \ldots, a_{v_r}, a_w \}$ (where $w \notin V$). On the other hand, each follower has $8$ actions per clause (each corresponding to a truth assignment to the variables in the clause), plus an additional one, namely $A = A_1 = A_2 = A_3 = \{ \phi_{ca} = l_1 l_2 l_3 \mid c \in \{ 1, \ldots, s \}, a \in \{ 1, \ldots, 8 \} \} \cup \{f\}$, where $\phi_{ca} = l_1 l_2 l_3$ identifies a truth assignment to the variables in $\phi_c$ such that $v(l_i)$ is set to true if and only if $l_i$ is a positive literal. For each follower $p \in F$, her payoffs are defined as follows:
	\begin{itemize}
		\item $U_{p,n}(\phi_{ca},a_{v_i}) = 1$ for all $v_i \in V$ and $\phi_{ca} \in A \setminus \{f\}$, with $v(l_p)=v_i$ and $l_p$ positive or $v(l_p) \neq v_i$ and $l_p$ negative;
		\item $U_{p,n}(\phi_{ca},a_{v_i}) = 0$ for all $v_i \in V$ and $\phi_{ca} \in A \setminus \{f\}$, with $v(l_p)=v_i$ and $l_p$ negative or $v(l_p) \neq v_i$ and $l_p$ positive;
		\item $U_{p,n}(\phi_{ca},a_w) = 0$ for all $\phi_{ca} \in A \setminus \{f\}$, if $l_p$ is positive, while $U_{p,n}(\phi_{ca},a_w) = 1$ otherwise;
		\item $U_{p,n}(f,a_n) = 0$ for all $a_n \in A_n$;
		\item $U_{p,q}(a_p,a_p) = 0$ for all $a_p \in A \setminus \{f\}$ and $q \in F \setminus \{p\}$;
		\item $U_{p,q}(a_p,a_q) = -1$ for all $a_p \in A \setminus \{f\}$, $a_q \neq a_p \in A \setminus \{f\}$, and $q \in F \setminus \{p\}$;
		\item $U_{p,q}(f,f) = 0$ and $U_{q,p}(f,f) = 1$ for all $p < q \in F$;
		\item $U_{p,q}(f,\phi_{ca}) = \frac{1}{2(r+1)}$ for all $\phi_{ca} \in A \setminus \{f\}$, with $l_p$ being a positive literal, while $U_{p,q}(f,\phi_{ca}) = \frac{r}{2(r+1)}$ if $l_p$ is negative, for $p < q \in F$;
		\item $U_{q,p}(\phi_{ca},f) = \frac{1}{2(r+1)}$ for all $\phi_{ca} \in A \setminus \{f\}$, with $l_q$ being a positive literal, while $U_{q,p}(\phi_{ca},f) = \frac{r}{2(r+1)}$ if $l_q$ is negative, for $p < q \in F$;
		\item $U_{q,p}(f,\phi_{ca})=0$ for all $\phi_{ca} \in A \setminus \{ f \}$ and $p < q \in F$;
		\item $U_{p,q}(\phi_{ca},f)=1$ for all $\phi_{ca} \in A \setminus \{ f \}$ and $p < q \in F$.
	\end{itemize}
	The payoffs for the leader are so defined:
	\begin{itemize}
		\item $U_{n,p}(\phi_{ca},a_n) = \frac{1}{3}$ for all $a_n \in A_n$ and $\phi_{ca} \in A \setminus \{f\}$ if the truth assignment identified by $\phi_{ca}$ makes $\phi_c$ true, while $U_{n,p}(\phi_{ca},a_n) = \frac{\epsilon}{3}$ otherwise, where $\epsilon > 0$;
		\item $U_{n,p}(f,a_n) = 1$ for all $a_n \in A_n$.
	\end{itemize}
	
	Initially, we prove the following lemma.
	\begin{lemma}\label{lem:not_poly_pes}
		For every $\phi_{ca} \in A \setminus \{f\}$, the outcome $(\phi_{ca},\phi_{ca},\phi_{ca})$ is an NE of the followers' game whenever the leader commits to a strategy $s_n \in \Delta_n$ satisfying the following constraints:
		\begin{itemize}
			\item $s_n(a_{v_i}) \geq \frac{1}{r+1}$ if $v(l_p) = v_i$ and $l_p$ is a positive literal, for some $p \in F$;
			\item $s_n(a_{v_i}) \leq \frac{1}{r+1}$ if $v(l_p) = v_i$ and $l_p$ is a negative literal, for some $p \in F$.
		\end{itemize}
		Moreover, all the outcomes of the followers' game that are not in $\{(\phi_{ca},\phi_{ca},\phi_{ca}) \mid \phi_{ca} \in A \setminus \{f\} \}$ cannot be played in a P-LFE, for any of the leader's commitments.
	\end{lemma}
	
	\begin{proof}
		Initially, we prove the first part of the statement.
		Let $s_n \in \Delta_n$ be an arbitrary leader's strategy. Then, for every $\phi_{ca} \in A \setminus \{f\}$, the outcome $(\phi_{ca},\phi_{ca},\phi_{ca})$ provides follower $p$ with the following utilities $U_p$:
		\begin{itemize}
			\item $U_p = s_n(a_{v_i})$ if $v(l_p) = v_i$ and $l_p$ is positive;
			\item $U_p = 1 - s_n(a_{v_i})$ if $v(l_p) = v_i$ and $l_p$ is negative.
		\end{itemize}
		Thus, by definition, $(\phi_{ca},\phi_{ca},\phi_{ca})$ is an NE if the following conditions hold:
		\begin{itemize}
			\item $U_p \geq \frac{1}{r+1}$ for each $p \in F$ such that $l_p$ is positive, as otherwise $p$ would deviate and play $f$;
			\item $U_p \geq \frac{r}{r+1}$ for each $p \in F$ such that $l_p$ is negative, as otherwise $p$ would deviate and play $f$.
		\end{itemize}
		This proves the first part of the statement.
		
		Notice that, for every leader's commitment $s_n \in \Delta_n$, there always exists at least one outcome $(\phi_{ca},\phi_{ca},\phi_{ca})$ which is an NE in the followers' game.
		
		Moreover, notice that all outcomes $(a_1,a_2,a_3)$ such that $ a_1,a_2,a_3  \in A \setminus \{f\}$ and $a_p \neq a_q$ for some $p,q \in F$ cannot be NEs since the followers get a negative payoff, while they can obtain a positive utility by deviating to $f$.
		Furthermore, the following outcomes cannot be NEs:
		\begin{itemize}
			\item $(f,f,f)$, as the first follower would deviate by playing any other action, thus increasing her utility from zero to something greater than or equal to 1;
			\item $(f,f,a)$, for any $a \in A \setminus \{f\}$, as the third follower would deviate playing action $f$, which guarantees her a utility of $2$ instead of something less than or equal to $1$;
			\item $(f,a,f)$, for any $a \in A \setminus \{f\}$, as the first follower would deviate to $a$, thus increasing her utility above $1$;
			\item $(a,f,f)$, for any $a \in A \setminus \{f\}$, as the second follower would deviate to $a$, thus increasing her utility above $1$;
			\item $(f,a,a)$, for any $a \in A \setminus \{f\}$, as the second follower would deviate to $f$, thus increasing her utility above $1$;
			\item $(a,f,a)$, for any $a \in A \setminus \{f\}$, as the third follower would deviate to $f$, thus increasing her utility above $1$.
		\end{itemize}
		Finally, all outcomes $(a,a,f)$, for all $a \in A \setminus \{f\}$, are never played by the followers in a P-LFE since, even if they could become NEs for some leader's commitment, they always provide the leader with a utility greater than 1, while, as previously shown, there is always at least another NE which gives her a utility at most equal to 1.
		\hfill $\Box$
	\end{proof}
	
	{\em YES-instance.} Suppose that the given 3-SAT instance has a YES answer, i.e., there exists a truth assignment which satisfies all the clauses. Then, by Lemma~\ref{lem:not_poly_pes}, there exists a strategy $s_n \in \Delta_n$ such that the worst (for the leader) NE in the followers' game provides her with a utility of $1$. Thus, the leader's utility in a P-LFE is $1$.
	
	{\em NO-instance.} Let us consider the case in which the instance has a NO answer. By Lemma~\ref{lem:not_poly_pes}, for every leader commitment $s_n \in \Delta_n$, there exists an NE in the followers' game that gives the leader a utility of $\epsilon$. Thus, the leader's utility in a P-LFE is $\epsilon$.
	
	{\em Contradiction.} Now, suppose there exists a polynomial-time approximation algorithm $\mathcal{A}$ with approximation factor $r = \frac{1}{f(n)}$, where $f(n)$ is any polynomial function of $n$. Moreover, let us fix $\epsilon = \frac{1}{2^n}$. If the 3-SAT instance has answer YES, then $\mathcal{A}$, when applied to the corresponding PG, must return a solution with value greater than or equal to $\frac{1}{f(n)} > \epsilon$. Instead, if the answer is NO, $\mathcal{A}$ must return a solution of value $\frac{\epsilon}{f(n)} < \epsilon$. Thus, the existence of $\mathcal{A}$ would imply that 3-SAT is solvable in polynomial time, which is an absurd, unless \textsf{P} = \textsf{NP}. 
	\hfill $\Box$
\end{proof}

\section{Conclusions and Future Works}

In this paper, we study the computational complexity of computing an O/P-LFE in two classes of polymatrix games that are of practical interest for security scenarios.
We show that the problem is Poly-$\mathsf{APX}$-complete and provide an exact algorithm to find a P-LFE for those game classes. 
These results can be extended to 2-player Bayesian games with uncertainty over the follower.
Finally, we show that in general polymatrix games computing an equilibrium is harder, even when players are forced to play pure strategies. 
In fact, in the optimistic case the problem is not in Poly-$\mathsf{APX}$ when the number of followers is non-fixed, while, in the pessimistic case, the same result also holds with only three followers.

Future works may develop along two directions. First, we could enhance our enumeration algorithm with a branch-and-bound scheme, following the approach of~\cite{jain2011quality}, which can only compute an O-LFE in Bayesian games.
Then, we could extend our results to other classes of succinct games, e.g., congestion games.

\bibliographystyle{aaai}
\bibliography{aaai_2018_pes}

\end{document}